\documentclass[twoside,12pt,reqno]{amsart}
\usepackage[foot]{amsaddr}
\usepackage[unicode=true,pdfusetitle, bookmarks=true,bookmarksnumbered=false,bookmarksopen=false, breaklinks=false,backref=false,colorlinks=false]{hyperref}
\hypersetup{colorlinks,linkcolor=myurlcolor,citecolor=myurlcolor,urlcolor=myurlcolor}
\usepackage{graphics,epstopdf,graphicx,amsthm,amsmath,amssymb,mathptmx,braket,colortbl,color,bm,framed,mathrsfs}
\usepackage[T1]{fontenc}
\usepackage[up]{subfigure}
\usepackage{tikz}
\usepackage{float}
\usepackage{amsfonts}
\usepackage{cite}
\usepackage{hyperref}

\usepackage[left=3.9cm, right=3.9cm, top=2cm]{geometry}

\usepackage[normalem]{ulem}

\usepackage{xcolor,cancel}

\usepackage{multirow,tabu}
\usepackage{array}
\usepackage{tabularx}
\newcolumntype{L}[1]{>{\raggedright\let\newline\\\arraybackslash\hspace{0pt}}m{#1}}
\newcolumntype{C}[1]{>{\centering\let\newline\\\arraybackslash\hspace{0pt}}m{#1}}
\newcolumntype{R}[1]{>{\raggedleft\let\newline\\\arraybackslash\hspace{0pt}}m{#1}}
\usepackage{tabu}

\usepackage{qcircuit}

\definecolor{myurlcolor}{rgb}{0,0,0.9}
\newcommand{\be}{\begin{equation}}
\newcommand{\ee}{\end{equation}}
\newcommand{\beq}{\begin{eqnarray}}
\newcommand{\eeq}{\end{eqnarray}}
\newcommand{\beqs}{\begin{eqnarray*}}
\newcommand{\eeqs}{\end{eqnarray*}}

\newcommand{\norm}[1]{\left\lVert #1 \right\rVert}

\theoremstyle{plain}
\newtheorem{thm}{Theorem}
\newtheorem{lem}[thm]{Lemma}
\newtheorem{prop}[thm]{Proposition}
\newtheorem{cor}[thm]{Corollary}
\theoremstyle{definition}
\newtheorem{Def}[thm]{Definition}

\tikzstyle WL=[line width=10pt,opacity=1.0]
\newcommand*{\myproofname}{Proof}

\def\real{\mathbb{R}}

\makeatother

\begin{document}

\title{Depth-Width Trade-offs for Neural Networks via Topological Entropy}
\author{Kaifeng Bu$^1$}
\email{kfbu@fas.harvard.edu (K.Bu)}
\author{Yaobo Zhang$^{2,3}$}
\email{yaobozhang@zju.edu.cn (Y.Zhang)}
\author{Qingxian Luo$^{4,5}$}
\email{luoqingxian@zju.edu.cn(Q.Luo)}

\address[$1$]{Department of Physics, Harvard University, Cambridge, Massachusetts 02138, USA}

\address[$2$]{Zhejiang Institute of Modern Physics, Zhejiang University, Hangzhou, Zhejiang 310027, China}
\address[$3$]{Department of Physics, Zhejiang University, Hangzhou Zhejiang 310027, China}

\address[$4$]{School of Mathematical Sciences, Zhejiang University, Hangzhou, Zhejiang 310027, China}
\address[$5$]{Center for Data Science, Zhejiang University, Hangzhou Zhejiang 310027, China}

\begin{abstract}

One of the central problems in the study of deep learning theory
 is to understand how the structure properties, such as depth, width and  the number of nodes,
 affect the  expressivity of deep neural networks. 
 In this work, we show a new connection between  the expressivity of deep neural networks and topological entropy
 from dynamical system, which can be used to  characterize depth-width trade-offs  of  neural networks. 
We provide an upper bound on the topological entropy of neural 
networks with continuous semi-algebraic units by the structure parameters. Specifically,  
the topological entropy of ReLU network with $l$ layers and $m$ nodes per layer is upper bounded by $O(l\log m)$. 
 Besides, if the neural network is a good approximation of some function $f$, then 
the size of the neural network has an exponential lower bound with respect to the topological 
entropy of $f$.
 Moreover, we  discuss the relationship between topological entropy,
the number of oscillations, periods and Lipschitz constant.

\end{abstract}

\maketitle

\section{Introduction}

Deep neural network has been a hot topic in machine learning, which has lots of applications ranging from 
pattern recognition to computer vision. 
Understanding the representation power  of neural network is one of the key problems in 
deep learning theory. Universal approximation theorem tells us that 
any continuous function can be approximated by a depth-2 neural network with some 
activation function on a bounded domain \cite{Cybenko1989,Hornik1989,Funahashi1989,Barron1994}. However, the size of the neural 
network in this approximation can be exponential which is impractical in real life. 
Hence, we are interested in the neural networks with bounded size. 

One natural question is to investigate the trade-offs between depth and width. The benefits of 
depths on the representational power of neural networks has attracted lots of attention, and there are many 
results based on the depth separation argument \cite{Eldan16,Telgarsky2015,Telgarsky2016,Schmitt1999,Montufar2014,Malach2019,Poole2016,Raghu17,Arora2016,Liang2016,Kileel2019}.
Depth separation argument has also been considered in other
computational models, such as boolean circuits \cite{Hastad1986,Hastad87, Parberry94,RossmanFOCS15}
and sum-product networks \cite{Shawe11,Martens2014}.
 To get a depth separation argument for neural networks, 
several measures to quantify the complexity of the functions have been
introduced, such as the number of 
linear regions \cite{Montufar2014},  Fourier spectrum \cite{Eldan16},
global curvature \cite{Poole2016},  
trajectory length \cite{Raghu17},  fractals \cite{Malach2019} and so on.

Recently, Telgarsky used the number of oscillations as a measure of 
the complexity of function to prove that there exist neural networks with 
$\theta(k^3)$ layers, $\theta(1)$ nodes per layer which can not be 
approximated by networks with $O(k)$ layers and $O(2^k)$ nodes \cite{Telgarsky2016}. 
Moreover, Chatziafratis et al provided a connection between 
the representation power of neural networks and the periods 
of the function by the well-known Sharkovsky's Theorem \cite{Chatziafratis2019}. 
Furthermore, by revealing a tighter connection between 
periods, Lipschitz constant and the  number of oscillations, 
Chatziafratis et al gave an  improved depth-width trade-offs \cite{Chatziafratis2020}. 

In this work, we show the connection between 
the representation power of neural networks and topological entropy,
a well-known concept in dynamic system to quantify the complexity of 
the system. 
First, we provide an upper bound on the topological entropy of neural 
networks with semi-algebraic units by the structure parameters like depth and 
width. For example, for the ReLU network with $l$ layers and $m$ nodes per layer, 
the topological entropy is upper bounded by $O(l\log m)$. 
Besides, if the neural network is a good approximation of  some function $f$, then 
the size has an exponential lower bound with respect to the topological 
entropy of $f$. Furthermore, we discuss the connection between topological entropy,
number of oscillations, periods and Lipschitz constant.

\section{Preliminaries}
\subsection{Background about dynamic system}
In this subsection, we will introduce some basic facts about one-dimensional dynamic system. 
First, let us introduce the definition of 
topological entropy. Topological entropy of a dynamic system
quantifies the complexity of the system, such as the  number of 
 different  orbits and the sensitivity of  evolution on the initial states. 
There are several equivalent definitions  of topological 
entropy. Here we take the one introduced by Adler, Konheim, and McAndrew \cite{Adler65}.

Let $X$ be a compact Hausdorff space, $f$ be a continuous map from 
$X$ to $X$. 
Given a set $\mathcal{A}$ of subsets of $X$,  if their union is $X$,
then $\mathcal{A}$  is called a cover of $X$. If each element in $\mathcal{A}$ is an open set, then $\mathcal{A}$
is called an open cover of $X$. 
Given open covers $\mathcal{A}_1, \mathcal{A}_2,...,\mathcal{A}_n$ of $X$, we denote $\bigvee^{n}_{i=1}\mathcal{A}_i$ as follows,
\begin{eqnarray*}
\bigvee^{n}_{i=1}\mathcal{A}_i:=
\set{
A_1\cap A_2...\cap A_n:
A_i\in \mathcal{A}_i , \forall i, ~\text{and}~~
A_1\cap A_2...\cap A_n\neq\emptyset
}.
\end{eqnarray*} 

Given an open cover 
$\mathcal{A}$, we can define the open 
cover $f^{-i}(\mathcal{A})$ and $\mathcal{A}^n_{f}$ as follows
\begin{eqnarray*}
f^{-i}(\mathcal{A})&:=&\set{f^{-i}(A): A\in \mathcal{A}},\\
\mathcal{A}^n_{f}&=&\bigvee^{n-1}_{i=0}f^{-i}(\mathcal{A}).
\end{eqnarray*}
Let us denote 
$\mathcal{N}(\mathcal{A})$ to be  the minimal cardinality of 
the subcover from $\mathcal{A}$. Mathematically,  $\mathcal{N}(\mathcal{A})$ can be defined as follows
\begin{eqnarray*}
\mathcal{N}(\mathcal{A})
=\min\set{Card(\mathcal{B}): \mathcal{B}\subset \mathcal{A} ~~\text{and} ~~\mathcal{B}~~ \text{is a cover of X}},
\end{eqnarray*}
where $Card(\mathcal{B})$ denotes the cardinality of $\mathcal{B}$.

Now, we are ready to define topological entropy. 

\begin{Def}\cite{Adler65}
Given a compact Hausdorff topological space $X$, and a continuous map $f:X\to X$, 
for an open cover $\mathcal{A}$,  the topological entropy of $f$
 on the cover $\mathcal{A}$ is defined as
 \begin{eqnarray*}
h_{top}(f,\mathcal{A})
=\lim_{n\to \infty}
\frac{1}{n}\log_2 \mathcal{N}(\mathcal{A}^{n}_f).
\end{eqnarray*}
The topological entropy of $f$ is defined as
\begin{eqnarray}
h_{top}(f)=\sup_{\mathcal{A}:~ \text{open cover of X}}
h_{top}(f,\mathcal{A}).
\end{eqnarray}

\end{Def}

The topological entropy takes value from $[0, +\infty]$.  (See Fig \ref{fig:infty} for the examples of functions with finite and infinite 
topological entropy.)
Topological entropy has some nice properties, which we have listed in the Appendix \ref{sec:top}.  
In this work, we  consider the case where $X$ is a closed interval $[a, b]$ and 
$f$ is a continuous function from $[a, b]$ to $[a,b]$.
For such interval map, topological entropy has several nice 
characterization. In this work, we will use the following one.
 We list other characterizations in  Appendix \ref{sec:top}.

 \begin{figure}[h]
  \center{
  \subfigure[]{
\begin{minipage}[t]{0.5\linewidth}
\centering
\includegraphics[width=5.5cm]{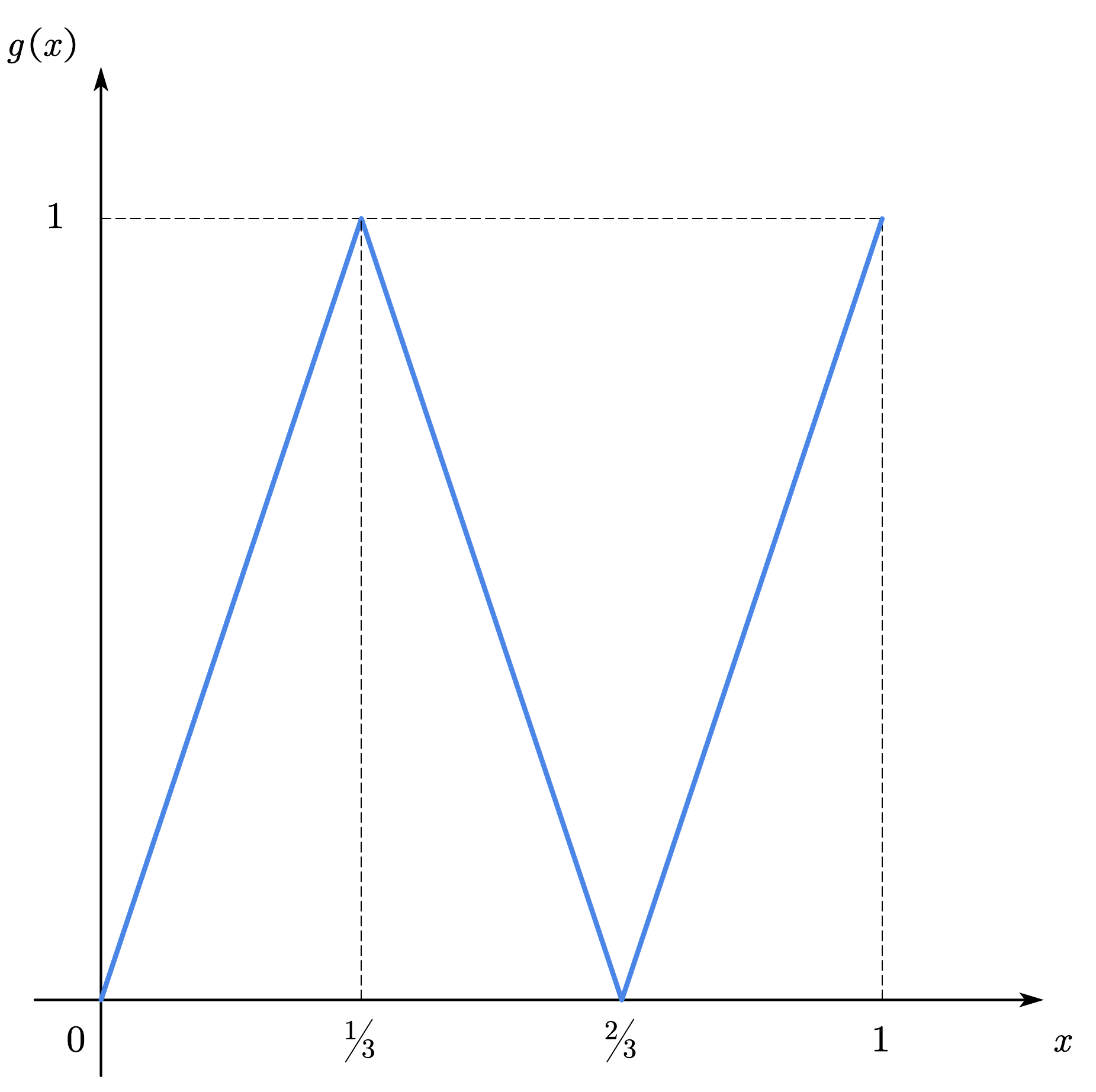}
\end{minipage}%
}%
\subfigure[]{
\begin{minipage}[t]{0.5\linewidth}
\centering
\includegraphics[width=5.5cm]{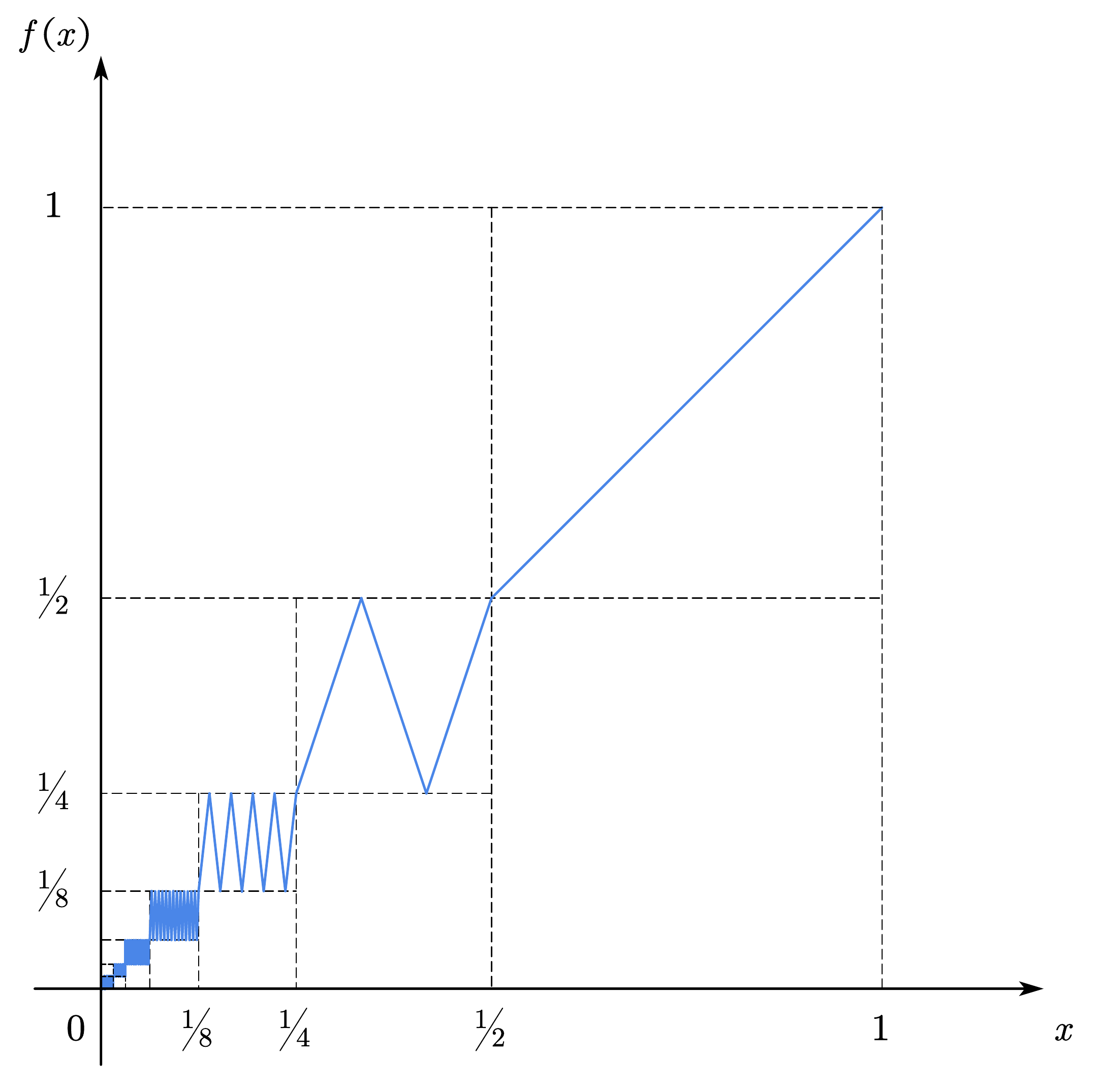}
\end{minipage}%
}%
  }     
  \caption{Examples of functions with finite and infinite topological entropy.
(a) $g:[0,1]\to [0,1]$ with $h_{top}(g)=3$;
(b) $f:[0,1]\to [0,1]$ with $h_{top}(f)=+\infty$, where $f$ is conjugate to $g^n$ on the interval $[2^{-(n-1)}, 2^{-n}]$ for each 
integer $n\geq 0$ and $f(0)=0$. (See the definition of conjugacy in Appendix \ref{sec:top}.)}
  \label{fig:infty}
 \end{figure}

\begin{Def}
A continuous function $f:[a,b]\to [a,b]$ is  piece-wise monotone, 
if there exists a finite partition of $[a,b]$ such that $f$ is monotone on each piece. Let us denote $c(f)$ 
to be minimal number of monotonicity of $f$.
\end{Def}

\begin{lem}\label{lem:monh}\cite{Misiurewicz80,Young81}
If the continuous function $f:[a,b]\to [a,b]$ is piecewise monotone, then 
\begin{eqnarray*}
h_{top}(f)
=\lim_{n\to \infty}
\frac{1}{k}\log c(f^k)=\inf_k \frac{1}{k}\log c(f^k),
\end{eqnarray*}
where $c(f)$ is the number of intervals of monotonicity of $f$.
\end{lem}

Now let us introduce the definition of periods in the dynamical system.

\begin{Def}
A continuous function $f:[a,b]\to [a,b]$ has a point of period $n$ if 
there exists $x_0\in [a,b]$ such that 
\begin{eqnarray*}
f^n(x_0)&=&x_0,\\
f^i(x_0)&\neq& x_0,~~ \forall 1\leq i\leq n-1.
\end{eqnarray*}
The set $\set{x_0, f(x_0),...,f^{n-1}(x_0)}$ is called 
a $n$-cycle of $f$.
\end{Def}

There is a well-known theorem called 
Sharkovsky's Theorem, which describes the structure of the periods of cycles
of the interval map. 

\begin{Def}[Sharkovsky's ordering]
Let us define Sharkovsky ordering as follows
\begin{eqnarray*}
&&3\vartriangleright 5 \vartriangleright7 \vartriangleright\cdots\vartriangleright\\
&\vartriangleright& 3\cdot 2\vartriangleright 5\cdot 2 \vartriangleright 7\cdot 2 \vartriangleright \cdots\vartriangleright\\
&\vartriangleright& 3\cdot 2^2\vartriangleright 5\cdot 2^2 \vartriangleright 7\cdot 2^2 \vartriangleright\cdots\vartriangleright\\
&&~~~~~~~~~~~~~~~~\vdots\\
&\vartriangleright& 3\cdot 2^n\vartriangleright 5\cdot 2^n \vartriangleright7\cdot 2^n \vartriangleright \cdots\vartriangleright\\
&&~~~~~~~~~~~~~~~~\vdots\\
&\vartriangleright& \cdots\vartriangleright  2^3 \vartriangleright 2^2 \vartriangleright 2\vartriangleright 1\\
\end{eqnarray*}
\end{Def}

Let us define $Per(f)$ to be the set of periods of cycles of a map $f:[a, b]\to [a, b]$ and denote $\mathbb{N}_{sh}=\mathbb{N}\cup \set{2^{\infty}}$.
Sharkovsky's Theorem tells us that  Sharkovsky's  ordering can be 
 used to characterize the  periods  of a continuous function  as follows.

\begin{thm}\cite{Sharkovsky64,Sharkovsky65}
Given a continuous function $f:[a, b]\to [a, b]$, there exists
$s\in \mathbb{N}_{sh}$ such that
$Per(f)=\set{k\in\mathbb{N}: s\vartriangleright  k}$. Conversely,
for any $s\in \mathbb{N}_{sh}$, there exists a continuous function 
$f:[a,b]\to [a,b]$ such that
$Per(f)=\set{k\in\mathbb{N}: s\vartriangleright  k}$.
\end{thm}

Next, let us give the definition of crossings (or oscillations), where the relationship between
the number of crossings and periods has been considered in \cite{Chatziafratis2019,Chatziafratis2020}.
\begin{Def}
Given a continuous function $f:[a,b]\to [a,b]$, for any $[x,y]\subset[a,b]$, 
$f$  crosses $[x,y]$ if there exists $c,d\in[a,b]$ such that 
$f(c)=x$, $f(d)=y$.  We use $C_{x,y}(f)$ to denote the number that $f$ crosses $[x,y]$, which means
there exists $c_1,d_1<c_2,d_2<...<c_t,d_t$ with $t=C_{x,y}(f)$ such that $f(c_i)=x, f(d_i)=y$ for any $1\leq i\leq C_{x,y}(f)$.
\end{Def}

Finally, let us introduce the concept called $f$-covering  \cite{Alseda00}.
\begin{Def}[$f$-covering]
Given a continuous function $f:[a,b]\to [a,b]$ and two intervals $I_1,I_2\subset [a,b]$, 
we say that $I_1$ $f$-covers $I_2$ 
if there exists a subinterval $J$ of $I_1$  such that 
$f(J)=I_2$. Besides,  we say that $I_1$ $f$-covers $I_2$ $t$ times if there 
exists $t$ subintervals $J_1,..,J_t $ of $I_1$ with pairwise disjoint interior such that 
$f(J_i)=I_2$ for $i=1,...,t$.
\end{Def}
Based on the definitions of crossing and $f$-covering, it is easy to see that 
$C_{x,y}(f)=t$ iff the maximal times that $[a,b]$ $f$-covers $[x,y]$ is equal to $t$.

\subsection{Neural networks with semi-algebraic units}
A neural network is a function defined by a connected directed graph with 
some activation function $\sigma: \real\to\real $ and a set of parameters:
a weight for each edge and a bias for each node of the graph. 
Usually the activation function   $\sigma: \real\to\real $ is a nonlinear function. 
The root nodes do the computation on the input vector, while the 
internal nodes do the computation on the output from other nodes.
The activation function for nodes may be different, and there are two common choices:
(1) ReLU gate: $\vec{x}\to \sigma_R(\langle\vec{a},\vec{x}\rangle+b)$, where $\sigma_R(x)=\max\set{0,x}$;
(2) maximaization gate $Max$: $\vec{x}\to \max^n_{i=1} x_i$.

Here we consider an important class of activation functions, called semi-algebraic units (or semi-algebraic gates)\cite{Telgarsky2016}. 
The definition of a semi-algebraic gate is given as follows
\begin{Def}
A function $\sigma: \real^n \to \real$ is called $(t, d_1, d_2)$ semi-algebraic, if there exists 
$t$ polynomials $\set{p_i}^{t}_{i=1}$ of degree $\leq d_1$ and 
$s$ tripes $(L_j, U_j, q_j)^{s}_{j=1}$ 
where $L_i$ and $U_i$ are subsets of $\set{1,2,....,t}$, and 
each $q_j$ is a polynomial of degree $\leq d_2$ 
such that 
\begin{eqnarray}
f(\vec{x})
=\sum^{s}_{j=1}
q_j(\vec{x}) \left(\Pi_{i\in L_j}
\mathbb{I}(p_i(\vec{x})<0)\right)
\left(\Pi_{i\in U_j}
\mathbb{I}(p_i(\vec{x})<0)\right),
\end{eqnarray}
where $\mathbb{I}(\cdot)$ is the indicator function.

\end{Def}

Here, we are interested in the continuous  semi-algebraic unit, that is the function $\sigma: \real^n \to \real$ is continuous
and semi-algebraic. 
For example,  the standard 
 ReLU gate $\vec{x}\to \sigma_R(\langle\vec{a},\vec{x}\rangle+b)$
 is a continuous and $(1,1,1)$ semi-algebraic unit \cite{Telgarsky2016}.
 The maximization gate $Max: \real^n\to \real$ defined as $Max(\vec{x})=\max^n_{i=1} x_i$
is a continuous and $(n(n-1),1,1)$ semi-algebraic unit \cite{Telgarsky2016}.

\begin{Def}
A function $\sigma:\real \to \real $ is called 
$(t, d)$-poly, if there exists a partition
of $\real$ into $\leq t$ intervals such that $\sigma$
is a polynomial of degree $\leq d$ on each interval.  
\end{Def}

Denote $\mathcal{N}_n(l,m,t,d_1,d_2)$ 
to be the set of neural networks with $\leq l$ layers,
$\leq m$ nodes per layer,  the activation function 
being continuous and $(t, d_1, d_2)$ semi-algebraic and the input dimension being $n$. 
As the function 
$f$ we would like to represent
is a continuous function $f:[a,b]\to [a,b]$, 
we consider the neural networks with 
input dimension being $1$, i.e., $\mathcal{N}_1(l,m,t,d_1,d_2)$.

\section{Informal statement of our main results}

Our first result shows the connection between topological entropy 
and the depth, width of deep neural networks, and provides an upper bound
of the topological entropy of  neural networks with continuous semi-algebraic units by the 
structure parameters.

\begin{thm}[Informal version of Theorem \ref{thm:main1}]
For any neural networks $g$ with 
$l$ layers, $m$ nodes per layer and 
$(t, d_1, d_2)$ semi-algebraic units 
as activation function, 
then 
\begin{eqnarray}
h_{top}(\tau\circ g )\leq l(1+\log_2m+\log_2t+\log_2d_1)
+l^2\log_2 d_2,
\end{eqnarray}
where
 $\tau: \real \to \real$ is defined as 
 follows (See Figure \ref{fig1}.)
\begin{equation}
\tau(x)=\left\{
\begin{array}{c}
a, x \leq 1,\\
x, a\leq x\leq b\\
b, x>b.
\end{array}
\right.
\end{equation}
\end{thm}

 \begin{figure}[!h]
  \center{\includegraphics[width=5.5cm]  {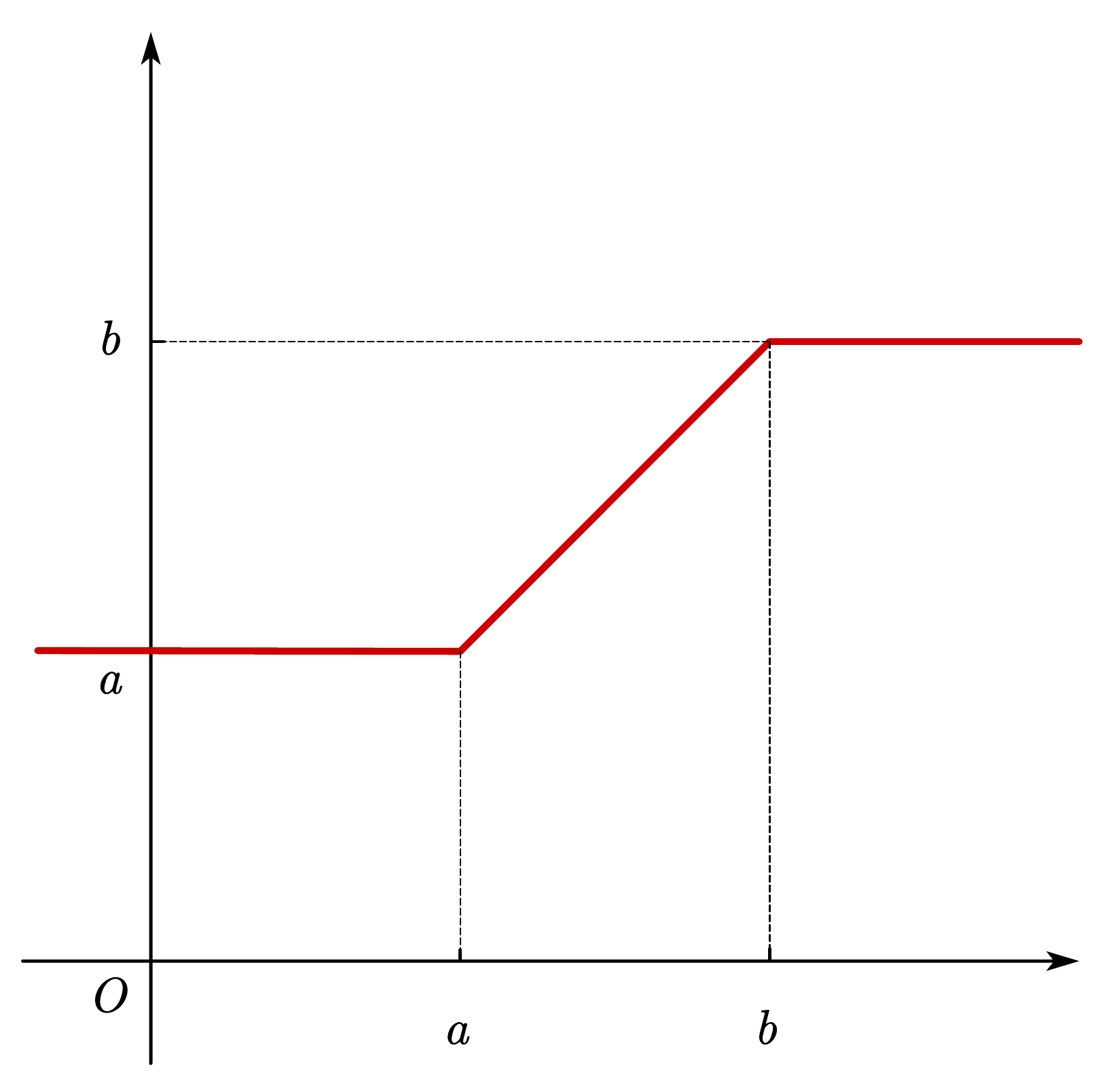}}     
  \caption{The figure for  the function $\tau(x)$.}
  \label{fig1}
 \end{figure}

Our second result shows the connection between the topological entropy of a given function $f$ and the 
depth-width trade-offs required to have a good approximation of $f$.

\begin{thm}[Informal statement of Theorem \ref{thm:main2}]
Given a continuous function
$f:[a,b]\to [a,b]$ with positive and finite topological entropy, 
if  $g$ is a good approximation of $f$ with respect to $\norm{\cdot}_{L^{\infty}}$, where
$g$ is a
neural network  with 
$l$ layers, $m$ nodes per layer and 
$(t, d_1, d_2)$ semi-algebraic units 
as activation function, then we have 
\begin{eqnarray}
m\geq \frac{exp(\Omega(\frac{1}{l}h_{top}(f)))}{2td_1d^l_2}.
\end{eqnarray}
Hence, if the neural network $g$ is a good approximation of $f^k$ with respect to $\norm{\cdot}_{L^{\infty}}$, then we have 
\begin{eqnarray}
m\geq \frac{exp(\Omega(\frac{k}{l}h_{top}(f)))}{2td_1d^l_2}.
\end{eqnarray}
\end{thm}

Our third result discusses the connection between the topological entropy, periods, the number of 
oscillations and Lipschitz constant.

\section{ Connection between topological entropy and the size of neural networks}
First, let us consider the topological entropy of 
the neural networks with $l$ layers, $m$ nodes per layer and 
activation function being $(t,d_1,d_2)$ semi-algebraic and continuous, i.e., the functions from
$\mathcal{N}_1(l,m,t,d_1,d_2)$. 
Let us define $\tau:\real\to\real$ as follows

\begin{equation*}
t_2(x)=\left\{
\begin{array}{ccc}
a,& x \leq 1,\\
x,& a\leq x\leq b,\\
b,& x>b.
\end{array}
\right.
\end{equation*}
We can rewrite $\tau(x)$ as follows
\begin{eqnarray}
\tau(x)=
a+
(x-a)\mathbb{I}(x>a)
+(b-x)\mathbb{I}(x>b).
\end{eqnarray}
Hence $\tau$ is continuous and $(2,1,1)$ semi-algebraic.
Therefore, for any $g\in\mathcal{N}_1(l,m,t,d_1,d_2) $, 
the function $\tau\circ g$ is a continuous function 
from $[a,b]$ to $[a,b]$. Thus, we can compute the 
topological entropy of 
 $\tau\circ g$.

To get an upper bound on the topological entropy of neural networks, we first 
need the following lemma, which gives an upper bound on the  number of intervals of monotonicity of $f$.

\begin{lem}\label{lem:pm}
If the function
$ f :[a,b]\to [a,b]$ is continuous and $(t,d)$-poly, we have 
\begin{eqnarray}
c(f)\leq td.
\end{eqnarray}
\end{lem}

\begin{proof}
Since $ f :[a,b]\to [a,b]$ is $(t,d)$-poly, then there exists a partition
of the interval $[a,b]$ into subintervals $\set{J_i}^t_{i=1}$ such that $f$ is a polynomial of degree
$\leq d$ on each subinterval $J_i$. 
It is directly for any polynomial degree
$\leq d$, we can divide $\real$ into 
$\leq d$ intervals such that this polynomial 
is monotone in each piece. 
Hence,  we can divide each 
subinterval $J_i$ into at most $d$ pieces, such that 
$f $ is monotone on each piece. 
Thus
\begin{eqnarray*}
c(f)\leq td.
\end{eqnarray*}
\end{proof}

Now, we are ready to prove our first result, which gives an upper bound on the topological entropy of the 
neural networks by the structure parameters of the neural networks.
\begin{thm}\label{thm:main1}
For any  $g\in \mathcal{N}_1(l,m,t,d_1,d_2)$, 
the 
topological entropy for the function $\tau\circ g:[a,b]\to [a,b]$
is upper bounded
by the structure parameters as follows
\begin{eqnarray}
h_{top}(\tau\circ g )
\leq
l(1+\log_2m+\log_2t+\log_2d_1)
+2l^2\log d_2.
\end{eqnarray}
\end{thm}
\begin{proof}
It has been proved that if the function 
$f:\real^n\to \real $ is $(t, d_1, d_2)$ semi-algebraic, 
$g_1,...,g_n:\real \to \real$ is $(s, d_3)$-poly, then 
$\mu(x):=f(g_1(x),...,g_n(x))$ is 
$(stn(1+d_1d_3), d_2d_3)$-poly \cite{Telgarsky2016}.
Thus, by analyzing the neural network  layer by layer, for any $g\in \mathcal{N}_1(l,m,t,d_1,d_2)$, 
$\tau\circ g $ is $(\alpha_l,\beta_l)$-poly, where
\begin{eqnarray*}
\alpha_l&\leq& 2(2mtd_1)^ld^{\frac{1}{2}l^2+l}_2,\\
\beta_l&\leq& d^l_2.
\end{eqnarray*}
Therefore, by Lemma \ref{lem:pm}, we have 
\begin{eqnarray*}
c(\tau\circ g )
\leq  2(2mtd_1)^ld^{2l^2}_2.
\end{eqnarray*}
By  Lemma \ref{lem:monh},
we have 
\begin{eqnarray*}
\lim_k\frac{1}{k}\log_2 c(f^k)
=\inf_k \frac{1}{k}\log_2 c(f^k)
=h_{top}(f),
\end{eqnarray*}
which implies that
\begin{eqnarray*}
c(f)\geq 2^{h_{top}(f)}.
\end{eqnarray*}
Therefore, 
 we have 
\begin{eqnarray*}
h_{top}(\tau\circ g)
\leq l(1+\log_2m+\log_2t+\log_2d_1)
+2l^2\log_2 d_2.
\end{eqnarray*}

\end{proof}

Next, 
to get the relationship between topological entropy 
of the function $f$ and that of the neural networks, we
need to consider the continuity of the topological entropy.

\begin{lem}\label{lem:semcon}\cite{Misiurewicz79hor}
For any continuous function 
$f:[a,b]\to [a,b]$, it holds that
\begin{eqnarray}
\lim_{g\to f}
\inf h_{top}(g)\geq 
h_{top}(f),
\end{eqnarray}
where $g:[a,b]\to [a,b]$ is continuous and $g\to f$ by  
$L^{\infty}$ norm.
\end{lem}

Based on the  lower semi-continuity of 
topological entropy,  if the 
given function has finite topological entropy, then for any $\epsilon>0$,
there exists $\delta>0$ such that 
for any continuous function $g:[a,b]\to [a,b]$ with 
$\norm{f-g}_{L^{\infty}}<\delta$, we have 
\begin{eqnarray*}
h_{top}(g)
\geq h_{top}(f)-\epsilon.
\end{eqnarray*}
If $0<h_{top}(f)<+\infty$, let us take $\epsilon=\frac{1}{2}h_{top}(f)$, there exists
$\delta(f)>0$ such that for any continuous function $g$ with $\norm{f-g}_{L^{\infty}}<\delta(f)$, we have 
\begin{eqnarray*}
h_{top}(g)
\geq \frac{1}{2}h_{top}(f).
\end{eqnarray*}

\begin{thm}\label{thm:main2}
Given a continuous function
$f:[a,b]\to [a,b]$ with positive and finite topological entropy, 
then there exists $\delta(f)>0$ such that 
for any  $g\in \mathcal{N}_1(l,m,t,d_1,d_2)$ with $\norm{f-g}_{L^{\infty}}\leq \delta(f)$, we have 
\begin{eqnarray}
m\geq 
\frac{2^{\frac{1}{2l}{h(f)}}}{2td_1d^{2l}_2}.
\end{eqnarray}
\end{thm}
\begin{proof}

First, based on Lemma \ref{lem:semcon}, 
there exists $\delta(f)>0$
such that for any continuous function 
$g:[a, b]\to [a,b]$, we have
\begin{eqnarray*}
h_{top}(\tau\circ g)
\geq \frac{1}{2}h_{top}(f).
\end{eqnarray*}

Besides, it is easy to see that $\tau$ is a Lipschitz function and 
$|\tau(x)-\tau(y)|\leq |x-y|$.
Hence, for any 
$g\in \mathcal{N}_1(l,m,t,d_1,d_2)$ with $\norm{f-g}_{L^{\infty}}\leq \delta(f)$, we have 
\begin{eqnarray*}
\norm{\tau\circ g-f}_{L^{\infty}}\leq \norm{g-f}_{L^{\infty}}\leq \delta(f).
\end{eqnarray*}
Then the topological entropy of $\tau\circ g:[a,b]\to [a,b]$ has the following 
lower bound,
\begin{eqnarray*}
h_{top}(\tau\circ g)
\geq \frac{1}{2}h_{top}(f).
\end{eqnarray*}

However, due to  Theorem \ref{thm:main1}, for any $g\in \mathcal{N}_1(l,m,t,d_1,d_2)$,
we have 
\begin{eqnarray*}
h(\tau\circ g )\leq1+l+l\log_2 m.
\end{eqnarray*}
Therefore, we have 
\begin{eqnarray*}
\frac{1}{2}h_{top}(f)
\leq  l(1+\log_2m+\log_2t+\log_2d_1)
+2l^2\log_2 d_2.
\end{eqnarray*}
That is 
\begin{eqnarray*}
m\geq 
\frac{2^{\frac{1}{2l}h_{top}(f)}}{2td_1d^{2l}_2}.
\end{eqnarray*}

\end{proof}
Theorem \ref{thm:main2} tells us that if 
the neural network $g\in \mathcal{N}_1(l,m,t,d_1,d_2)$ 
is a good approximation (i.e., $\norm{f-g}_{L^{\infty}}\leq \delta(f)$ ), 
then the depth $m$ has an exponential lower bound with respect to 
the topological entropy.

Besides, if  we iterate the function for $k$ times, i.e, 
$f^k$ and the neural network $g\in \mathcal{N}_1(l,m,t,d_1,d_2)$  is a good approximation of 
$f^k$, we have the following corollary.

\begin{cor}
Given a continuous function
$f:[0,1]\to [0,1]$ with positive and finite topological entropy, 
then there exists $\delta(f^k)>0$ such that 
for any  $g\in \mathcal{N}_1(l,m,t,d_1,d_2)$ with $\norm{f^k-g}_{L^{\infty}}\leq \delta(f^k)$, we have 
\begin{eqnarray}
m\geq 
\frac{2^{\frac{k}{2l}h_{top}(f)}}{2td_1d^{2l}_2}.
\end{eqnarray}
\end{cor}
\begin{proof}
This corollary comes directly from 
Theorem \ref{thm:main2} and 
$h_{top}(f^k)=kh_{top}(f)$ for any integer $k\geq 0$. (See Lemma \ref{lem:k}
in Appendix \ref{sec:top}.) 
\end{proof}

For example, if we take the activation function to be
ReLU unit  which is continuous and (1,1,1) semi-algebraic, then the following 
statements come directly from Theorem \ref{thm:main1} and  \ref{thm:main2}.

\begin{prop}
For any ReLU network $g$ with at most $l$ layers and  at most $m$ nodes per layer, then
\begin{eqnarray}
h_{top}(\tau\circ g)
\leq l(1+\log_2m).
\end{eqnarray}
\end{prop}

\begin{prop}
Given a continuous function
$f:[a,b]\to [a,b]$ with finite topological entropy, 
then there exists $\delta(f)>0$ such that 
for any ReLU network  $g$ with at most $l$ layers and  at most $m$ nodes per layer which satisfies  
$\norm{f-g}_{L^{\infty}}\leq \delta(f)$, we have 
\begin{eqnarray}
m\geq 2^{\frac{1}{2l}h_{top}(f)-1}.
\end{eqnarray}

Moreover, if $g$ is 
 a good approximation of 
$f^k$ with respect to $L^{\infty}$ norm, i.e., $\norm{f^k-g}_{L^{\infty}}\leq \delta(f^k)$, then 
we have 
\begin{eqnarray}
m\geq 2^{\frac{k}{2l}h_{top}(f)-1}.
\end{eqnarray}

\end{prop}

If the  function $f$ we would like to present has infinity topological entropy, i.e., $h_{top}(f)=+\infty$, 
then due to the lower semi-continuity 
of topological entropy, for any $N>0$, there exists $\delta_{N}(f)>0$ such that 
for any continuous function $g: [a,b]\to [a,b]$
with $\norm{g-f}_{L^{\infty}}<\delta_N(f)$,
\begin{eqnarray}
h_{top}(g)\geq N
\end{eqnarray}
\begin{prop}
Given a continuous function
$f:[a,b]\to [a,b]$ with $h_{top}(f)=+\infty$,  then 
any $N>0$ sufficiently large, there exists
$\delta_N(f)>0$ such that 
for any $g\in \mathcal{N}_1(l,m,t,d_1,d_2)$ with $\norm{f-g}_{L^{\infty}}<\delta_N(f)$, we have 
\begin{eqnarray}
m\geq 
\frac{2^{N/l}}{2td_1d^l_2}.
\end{eqnarray}
\end{prop}
\begin{proof}
The proof is the same as Theorem \ref{thm:main2}.
\end{proof}

\subsection{Examples}

First, 
let us consider the tent map $t_{\alpha}:[0,1]\to [0,1]$, where $t_{\alpha}(x)$ is defined as follows
\begin{equation*}
t_{\alpha}(x)=\left\{
\begin{array}{c}
\alpha x, 0\leq x\leq 1/2,\\
\\
\alpha(1-x), 1/2<x\leq 1,
\end{array}
\right.
\end{equation*}
where $0\leq \alpha\leq 2$. (See Figure \ref{fig:tent})

 \begin{figure}[h]
  \center{
  \subfigure[Tent map $t_{\alpha}$]{
\begin{minipage}[t]{0.5\linewidth}
\centering
\includegraphics[width=5.5cm]{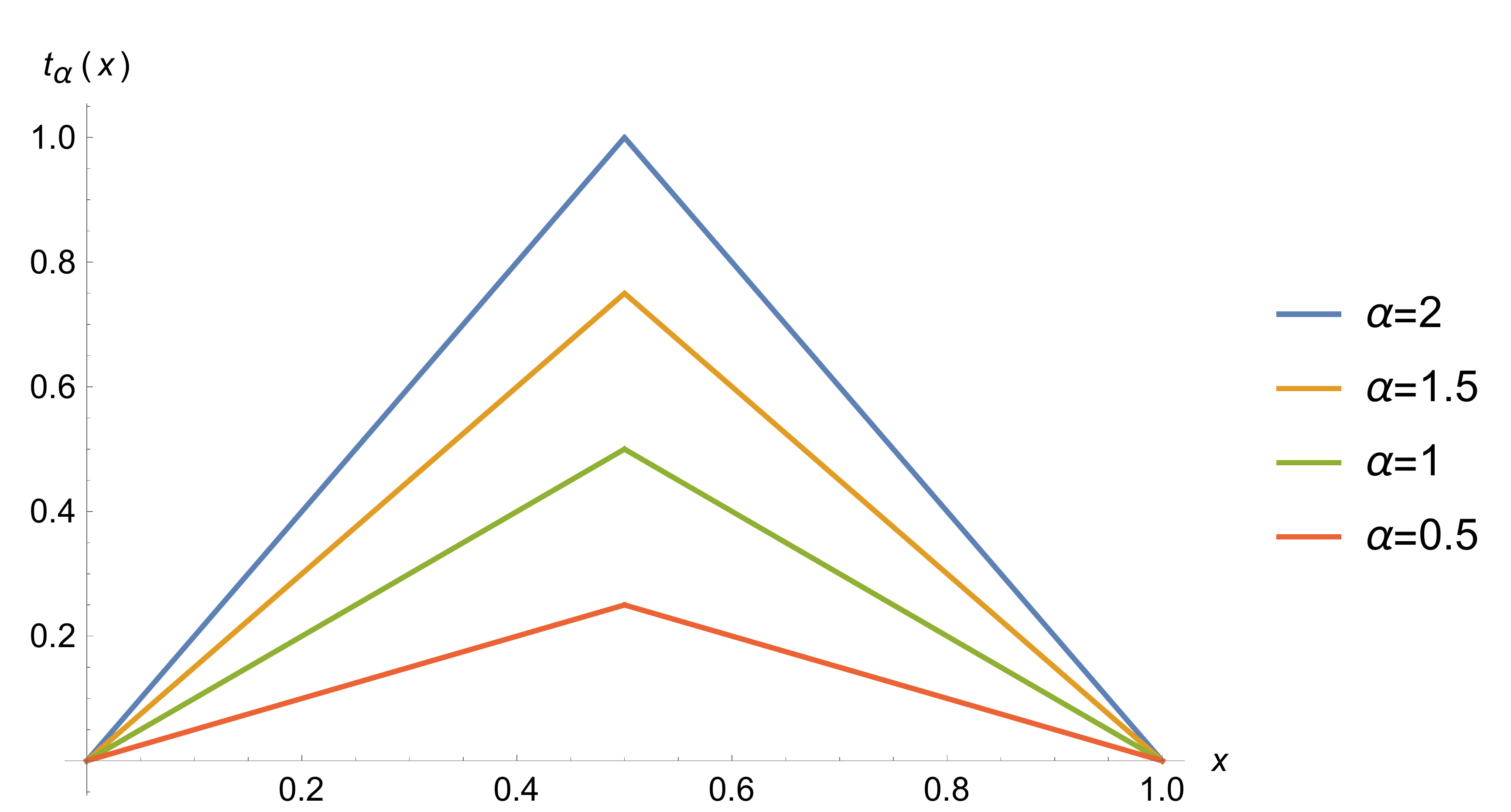}
\end{minipage}%
}%
\subfigure[$t^4_{\alpha}$]{
\begin{minipage}[t]{0.5\linewidth}
\centering
\includegraphics[width=5.5cm]{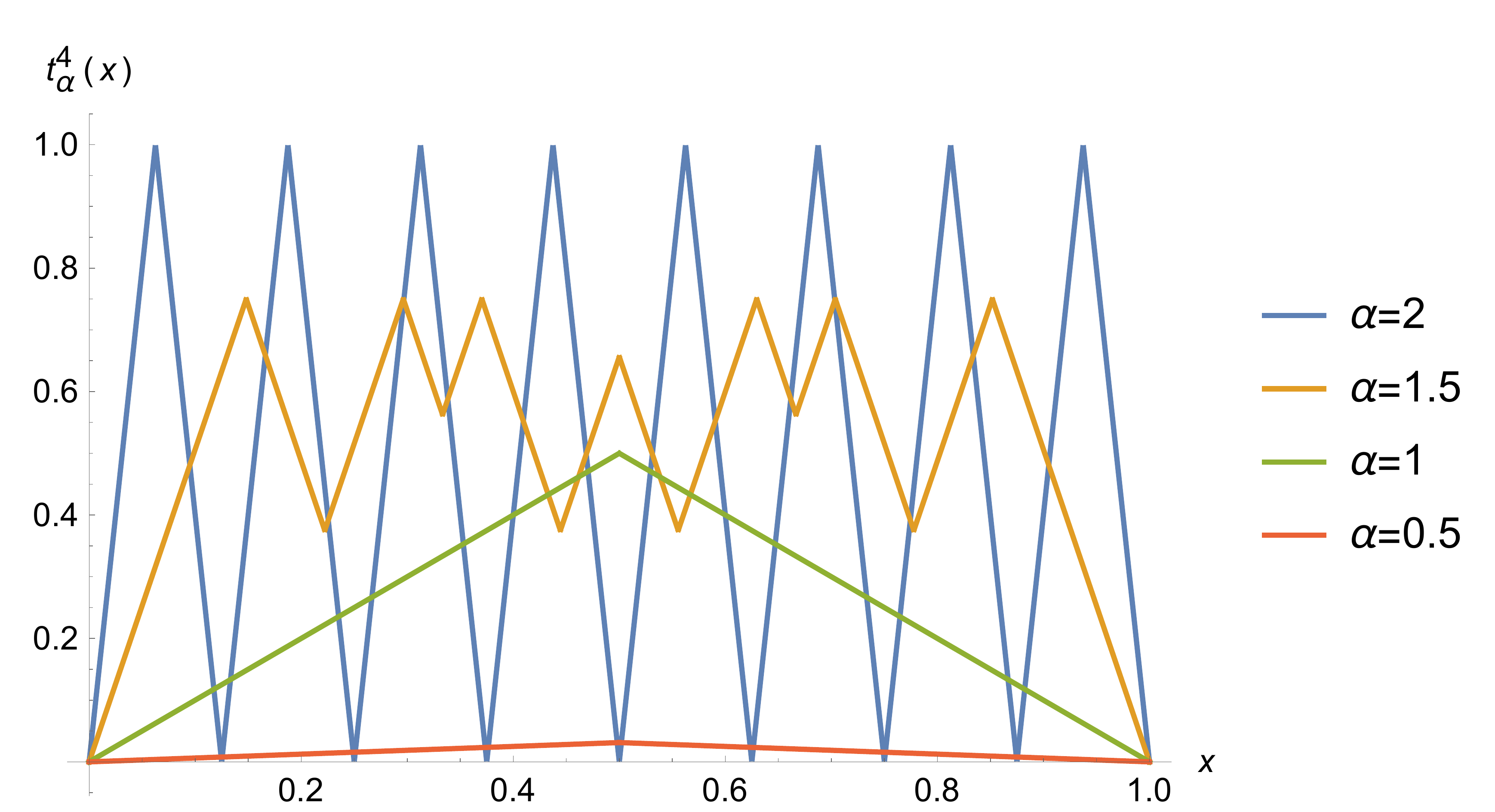}
\end{minipage}%
}%
  }     
  \caption{Tent map $t_{\alpha}$ and $t^4_{\alpha}$  with  different parameters $\alpha$.}
  \label{fig:tent}
 \end{figure}

The topological entropy of $t_{\alpha}$ can be easily computed by Lemma \ref{lem:slope}, and we have 
\begin{equation*}
h_{top}(t_{\alpha})=\left\{
\begin{array}{c}
0, 0\leq \alpha \leq 1,\\
\\
\log_2\alpha, 1<\alpha\leq 2.
\end{array}
\right.
\end{equation*}
(See Figure \ref{fig:toptent}.)

 \begin{figure}[h]
  \center{\includegraphics[width=5.5cm]  {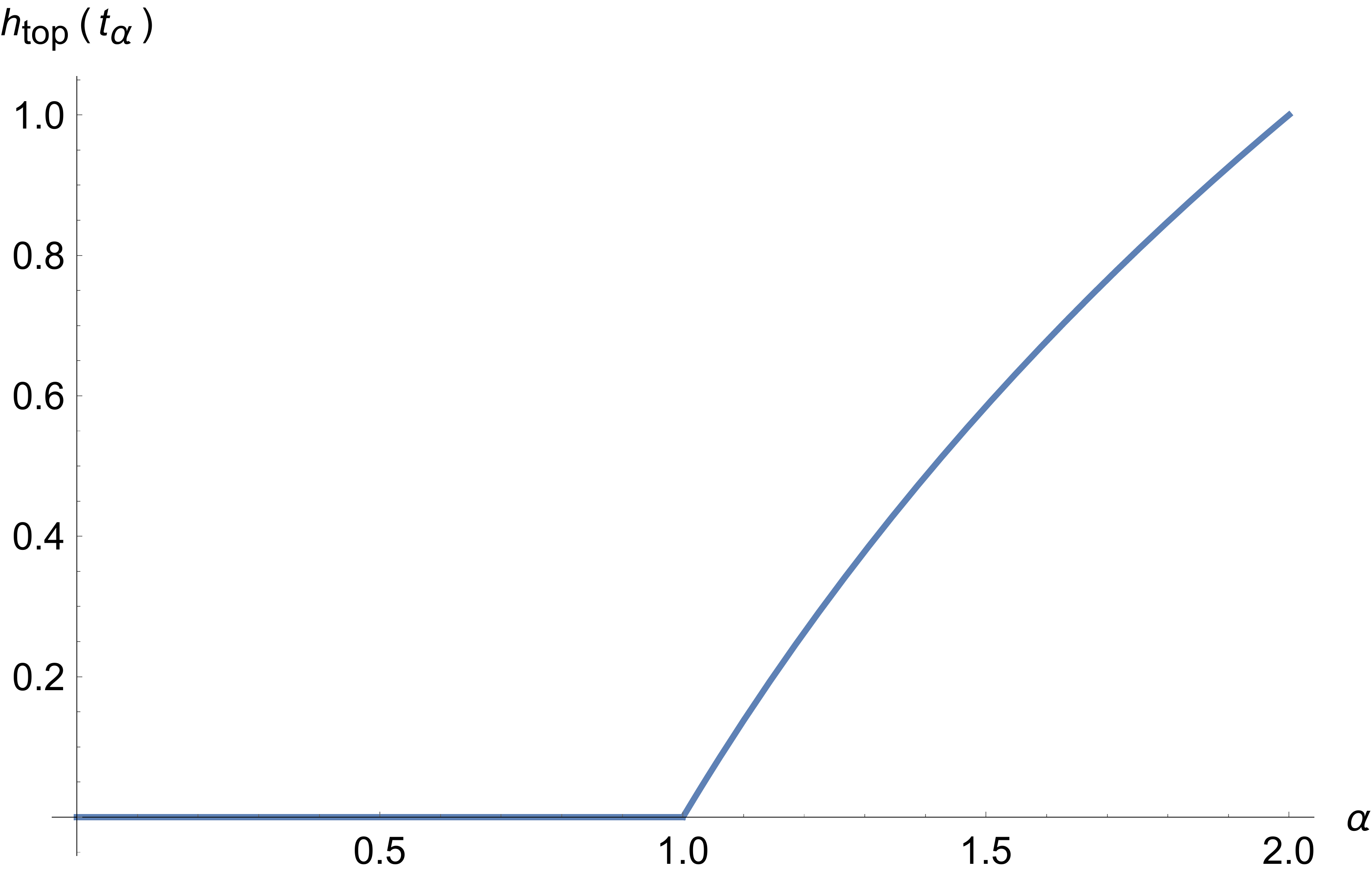}}     
  \caption{The topological entropy of the tent map $t_{\alpha}$ for $0<\alpha\leq 2$.}
  \label{fig:toptent}
 \end{figure}

Hence,   based on Theorem \ref{thm:main2}, if we would  like to  have a good approximation 
of $t^{k}_{\alpha}$ for $\alpha>1$, then 
the width required to represent $t^{k}_{\alpha}$ with continuous and  $(t,d_1,d_2)$ semi-algebraic units is
\begin{eqnarray*}
m\geq C(t,d_1,d_2)\alpha^{k/l},
\end{eqnarray*}
where $C(t,d_1,d_2)$ is a constant which only depends on 
$t,d_1,d_2$.

Next, 
let us consider the logistic map
$f_{\beta}:[0,1]\to [0,1]$ as follows
\begin{eqnarray*}
f(x)=\beta x(1-x),
\end{eqnarray*}
where the parameter $\beta$ is taken from $[0,4]$ (See Figure \ref{fig:log}). 
Logistic map has been used  to get   lower bounds on the size of sigmoidal neural networks \cite{Schmitt1999}.

 \begin{figure}[h]
  \center{
  \subfigure[Logistic map $f_{\beta}$]{
\begin{minipage}[t]{0.5\linewidth}
\centering
\includegraphics[width=5.5cm]{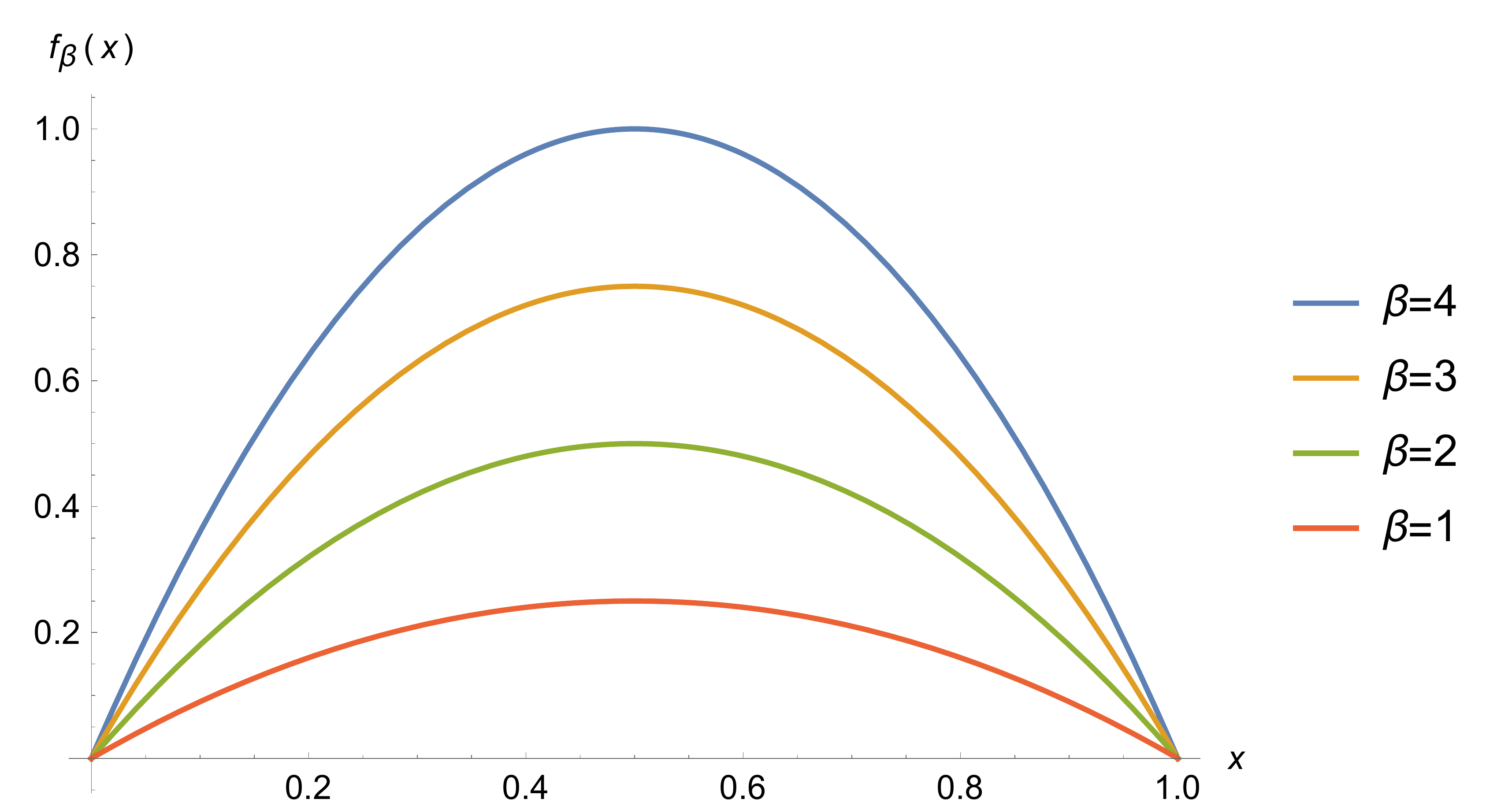}
\end{minipage}%
}%
\subfigure[$f^4_{\beta}$ ]{
\begin{minipage}[t]{0.5\linewidth}
\centering
\includegraphics[width=5.5cm]{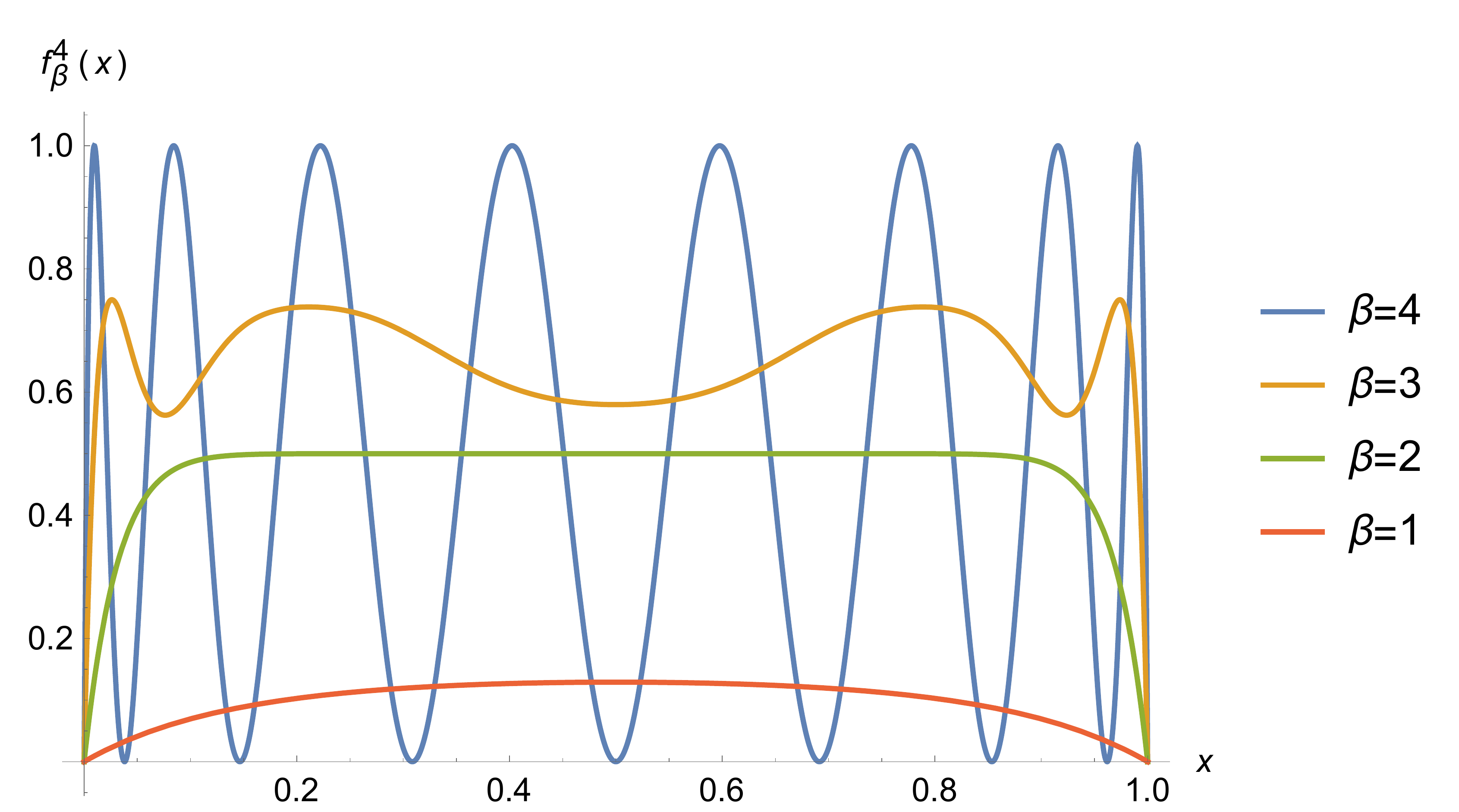}
\end{minipage}%
}%
  }     
  \caption{Logistic map $f_{\beta}$ and $f^4_{\beta}$ with different parameters $\beta$.}
  \label{fig:log}
 \end{figure}

It is easy to see that $ h_{top}(f_\beta)=1$ when $\beta=4$, and  $h_{top}(f_\beta)=0$ when 
$\beta=2$. Hence,   based on  Theorem \ref{thm:main2}, if we would  like to  have a good approximation 
of $f^{k}_{4}$, then 
the width required to represent $f^{k}_{4}$ with continuous and $(t,d_1,d_2)$ semi-algebraic function is
\begin{eqnarray*}
m\geq C(t,d_1,d_2)2^{k/l},
\end{eqnarray*}
where $C(t,d_1,d_2)$ is a constant which only depends on 
$t,d_1,d_2$.

\section{Relationship between topological entropy and periods, number of crossings and Lipschitz constant }
In this section, we will discuss the connection between 
topological entropy and periods, number of crossings and Lipschitz constant.
\subsection{Relationship between topological entropy and periods, the number of crossings }
 In fact, the relationship between topological entropy and periods has been discussed in \cite{Alseda00}, which has the following statement. 

\begin{lem}[\cite{Alseda00}]
Given a continuous map $f:[a,b]\to [a,b]$, it has positive topological entropy
iff it has a cycle of 
period which is not a power of 2. 
\end{lem}

In this subsection, we will show the connection between 
topological entropy and the number of crossings  for piece-wise monotone function
$f:[a,b]\to [a,b]$.   
Let us define $C(f)$ as follows
\begin{eqnarray}
C(f):=\sup_{x<y} C_{x,y}(f),
\end{eqnarray}
which is the
 maximal number of 
crossings over any interval $[x,y]\subset[a,b]$. We find the relationship between 
the maximal number of crossings $C(f)$  and 
topological entropy $h_{top}(f)$ in the asymptotic case.

\begin{prop}
Given a continuous function $f:[a,b]\to [a,b]$ which is piece-wise monotone, 
then 
\begin{eqnarray}
\lim_{k\to\infty}\sup_k
\frac{1}{k}\log_2 C(f^k)
=h_{top}(f).
\end{eqnarray}
\end{prop}

\begin{proof}

First,  since 
$f$ is piece-wise monotone, then 
there exists a finite partition of $[a,b]$
into subintervals such that 
$f$ is monotone on each subinterval. 
For any  subinterval where $f$ is monotone, 
there is at most one crossing over $[x,y]$. 
Thus
for any $x,y\in [a,b]$, we have 
\begin{eqnarray*}
C_{xy}(f)
\leq c(f),
\end{eqnarray*}
i.e., $C(f)\leq c(f)$.
Therefore, 
\begin{eqnarray*}
\lim_k \sup_k
\frac{1}{k}\log_2 C(f^k)
\leq \lim_k \frac{1}{k}\log_2 c(f^k)
=h_{top}(f).
\end{eqnarray*}

Besides,
if $h_{top}(f)=0$, then we have already got the result as
\begin{eqnarray*}
\lim_{k\to\infty}\sup_k
\frac{1}{k}\log_2 C(f^k)\geq 0
\end{eqnarray*}
Hence, we only need to consider the case where
$h_{top}(f)>0$. 
Let us introduce the concept called $s$-horeses \cite{Misiurewicz79hor,Misiurewicz80hor},
which is 
 an interval $J\subset [a,b]$  and  a partition $\mathcal{D}$ of 
$J$ into s subintervals such that 
 the closure of each element of $\mathcal{D}$
$ f$-covers $J$.
It  has been proved in  \cite{Misiurewicz79hor,Misiurewicz80hor} that there exist sequences $\set{k_n}^{\infty}_{n=1}$ and 
$\set{s_n}^{\infty}_{n=1}$ of positive integers such that 
$\lim_{n\to \infty}k_n=\infty$ and for each $n$, there exists $s_n$-horseshoes $(J_n, D_n)$ for $f^{k_n}$ such that 
\begin{eqnarray*}
\lim_{n\to \infty}\frac{1}{k_n}
\log_2 s_n
=h_{top}(f).
\end{eqnarray*}

 Based on the definition of $s_n$-horseshoe, for the map $f^{k_n}$, 
the closure of each  subinterval in $D_n$ $f^{k_n}$-covers $J_n$.
Thus, based on the definition of crossings, we have 
\begin{eqnarray*}
C(f^{k_n})\geq C_{J_n}(f^{k_n})\geq
s_n.
\end{eqnarray*}

Therefore
\begin{eqnarray*}
\lim_k \sup_k
\frac{1}{k}\log_2 C(f^k)
\geq \lim_{n\to \infty}\frac{1}{k_n}
\log_2 s_n
=h_{top}(f).
\end{eqnarray*}

\end{proof}

\subsection{Relationship between topological entropy and Lipschitz constant}

Let us consider the connection between Lipschitz constant and topological entropy.
Let us  denote the Lipschitz constant of $f$ by $L(f)$, that is 
\begin{eqnarray}
L(f)
=\inf\set{L\geq 0: |f(x)-f(y)|\leq L|x-y|, \forall x,y\in[a,b]}.
\end{eqnarray}

The connection between periods, the number of crossings and   Lipschitz constant has been discussed in \cite{Chatziafratis2020}. 
It has been proved that
if the Lipschitz constant matches the number of crossings, i.e.,  $C_{x y}(f^k)=L(f^k)$, then 
a $L^1$-separation  between $f^k$ and ReLU neural networks  can be obtained \cite{Chatziafratis2020}. 
Here we discuss the relationship between  Lipschitz constant and 
topological entropy.

\begin{prop}\label{prop:LipT}
Given a  continuous function $f:[a,b]\to [a,b]$ which piece-wise monotone, 
then 
\begin{eqnarray}
\lim_{k\to\infty}
\frac{1}{k}\log_2 L(f^k)
=\inf_k \frac{1}{k}\log_2 L(f^k),
\end{eqnarray}
and 
\begin{eqnarray}
\lim_{k\to \infty}
\max \set{0, \frac{1}{k}\log_2 L(f^k)}
\geq h_{top}(f).
\end{eqnarray}
\end{prop}

\begin{proof}
Based on the definition of Lipschitz constant, 
it is easy to see that 
\begin{eqnarray*}
|f^{n+k}(x)-f^{n+k}(y)|
&=&|f^{n}(f^{k}(x))-f^{n}(f^{k}(y))|\\
&\leq& L(f^n) |f^k(x)-f^k(y)|\\
&\leq&  L(f^n) L(f^k) |x-y|,
\end{eqnarray*}
for any integers $n,k$ and any $x,y\in [a,b]$.
Thus,
\begin{eqnarray}
L(f^{n+k})
\leq L(f^n) L(f^k).
\end{eqnarray}
i.e., $\log_2 L(f^{n+k})\leq\log_2 L(f^{n})+\log_2 L(f^{k})$. 
Hence $\set{\log_2 L(f^k)}_k$
 is a subadditive sequence. 
 Therefore, according to 
 Lemma \ref{lem:sub} in Appendix \ref{sec:top}, the limit
 \begin{eqnarray*}
 \lim_{k\to\infty}\frac{1}{k}\log_2 L(f^k)
 \end{eqnarray*}
exists and  
\begin{eqnarray*}
\lim_{k\to \infty}\frac{1}{k}\log_2 L(f^k)
=\inf_{k}\frac{1}{k}\log_2 L(f^k).
\end{eqnarray*}
Let us another characterization of  topological entropy of the function, which is  piece-wise monotone,  by 
variation  \cite{Alseda00} as follows 
\begin{eqnarray*}
\lim_{k\to \infty}
\max \set{0, \frac{1}{k}\log_2 Var(f^k)}
=h_{top}(f),
\end{eqnarray*}
where variation $Var(f)$ is defined to be 
the supremum of 
\begin{eqnarray*}
\sum^{t}_{i=1}
|f(x_{i+1})-f(x_i))|,
\end{eqnarray*}
over all finite sequences $x_1<x_2<....<x_t$ in $[a,b]$.
(See Lemma \ref{lem:var} in Appendix \ref{sec:top}.) 
Due to the definition 
of $Var(f)$, it is easy to see 

\begin{eqnarray*}
Var(f^k)
\leq L(f^k)|b-a|.
\end{eqnarray*}
Therefore, 
\begin{eqnarray*}
\lim_{k\to \infty}\frac{1}{k}\log_2 L(f^k)
\geq \lim_{k\to \infty}\frac{1}{k}\log_2 Var(f^k),
\end{eqnarray*}
which implies that
\begin{eqnarray*}
\lim_{k\to \infty}
\max \set{0, \frac{1}{k}\log_2 L(f^k)}
\geq h_{top}(f).
\end{eqnarray*}

\end{proof}

Based on Proposition \ref{prop:LipT}, if $L(f^k)\geq 1$, then $L(f^k)$ has 
an exponential lower bound with respect to the topological entropy of 
$f$, i.e., 
$L(f^{k})\geq 2^{kh_{top}(f)}$.

\section{Conclusion}
In this paper, we have investigated the relationship between topological entropy and 
expressivity of deep neural networks.
We provide a depth-width 
trade-offs based on the topological entropy from the theory of dynamic system.
For example, 
the topological entropy of  the ReLU network with $l$ layers and $m$ nodes per layer is upper bounded by $O(l\log m)$. 
Besides, we show that the size of the neural network required to represent a given function has an exponential lower bound with respect to
the topological entropy of the function, where
the exponential lower bound holds for
$L^{\infty}$-error approximation. 
For example, if we would like to represent the function $f$ by ReLU network with $l$ layers and $m$ nodes per layer, 
then the width $m$ has a lower bound $\exp(\Omega(h_{top}(f)/l))$.
Moreover, we discuss the relationship between 
topological entropy, periods, Lipschitz constant and  the number of crossings,
especially the relationship in the asymptotic case. 

Note that one key step to get exponential lower bound  on the size of neural networks for  $L^{\infty}$-error approximation 
is the lower semi-continuity of 
topological entropy with respect to 
$L^{\infty}$ norm. 
If the lower semi-continuity of 
topological entropy with respect to $L^p$ norm (e.g., $L^1$ norm) holds, it  
will lead to exponential lower bound  (with respect to topological entropy) 
for $L^p$-error approximation. 
Further studies on 
the (semi-)continuity of topological entropy are desired.
Besides, 
it would be quite interesting to study the
relationship between topological entropy and 
VC dimension. We leave it for further study.

\section{Acknowledgments}
K. B. thanks Arthur Jaffe for the help and support and thanks Weichen Gu for the discussion on topological entropy. 
K. B. acknowledges the support of
 ARO Grants W911NF-19-1-0302 and
W911NF-20-1-0082.

\begin{appendix}

\section{Properties of topological entropy}\label{sec:top}

Here, we list some useful facts about topological entropy. More information can be found in
\cite{Alseda00}.

\begin{lem}\cite{Alseda00}\label{lem:k}
Given a compact Hausdorff space $X$ and a continuous function $f: X\to X$,  
topological entropy of $f$ and $f^{k}$ has the following relatiobship

\begin{eqnarray}
h_{top}(f^k)
=kh_{top}(f),
\end{eqnarray}
for any integer $k\geq 0$.
\end{lem}

\begin{prop}\cite{Alseda00}
Given compact Hausdorff spaces $X,Y$,
$f:X\to X, g:Y\to Y, \phi:X\to Y$ are continuous maps such that the following diagram
 \begin{equation}
\begin{array}{cccc}
X&\stackrel{f}{\longrightarrow}&X\\
\small{\varphi}\downarrow~&~~~~~~~~~~~~~~&~\downarrow\small{\varphi}\\
Y&\stackrel{g}{\longrightarrow}&Y
\end{array}
\end{equation}
commutes, i.e., $\varphi\circ f=g\circ \varphi$, we have  the following properties 

(a) if $\varphi$ is injective, then $h_{top}(f)\leq h_{top}(g)$,

(b) if $\varphi$ is surjective, then $h_{top}(f)\geq h_{top}(g)$,

(c) if $\varphi$ is bijective, then $h_{top}(f)= h_{top}(g)$. And $\varphi$ is called a conjugacy between 
$f$ and $g$ (or $f$ and $g$ are conjugate).

\end{prop}

If $X=[a, b]$,  
then we have the following 
characterization of topological entropy for a continuous function $f:[a,b]\to [a,b]$.

\begin{Def}[\cite{Misiurewicz79hor,Misiurewicz80hor}]
Given a continuous function $f:[a,b]\to [a,b]$,  an s-horseshoe with $s\geq 2$ for $f$ is 
$(J, \mathcal{D})$, where
 $J\subset [a,b]$ is an interval and  $\mathcal{D}$ is a partition
$J$ into s subintervals such that 
 the closure of each element of $\mathcal{D}$
$ f$-covers $J$.
\end{Def}

\begin{lem}[\cite{Misiurewicz79hor,Misiurewicz80hor}]
Given a continuous function $f:[a,b]\to [a,b]$ with positive entropy, then there exist sequences $\set{k_n}^{\infty}_{n=1}$ and 
$\set{s_n}^{\infty}_{n=1}$ of positive integers such that 
$\lim_{n\to \infty}k_n=\infty$, for each $n$ the map $f^{k_n}$ has an $s_n$-horseshoe and
\begin{eqnarray}
\lim_{n\to \infty}
\frac{1}{k_n}
\log s_n=h_{top}(f).
\end{eqnarray}
\end{lem}

\begin{Def}\cite{Alseda00}
Given a continuous function $f:[a,b]\to [a,b]$, the variation $Var(f)$ is defined to be 
the supremum of 
\begin{eqnarray}
\sum^{t}_{i=1}
|f(x_{i+1})-f(x_i))|
\end{eqnarray}
over all finite sequences $x_1<x_2<....<x_t$ in $[a,b]$.
\end{Def}

\begin{lem}\cite{Alseda00}\label{lem:var}
Given a continuous function $f:[a,b]\to [a,b]$ which piece-wise monotone, then we have
\begin{eqnarray}
\lim_{k\to \infty}
\max \set{0, \frac{1}{k}\log_2 Var(f^k)}
=h_{top}(f).
\end{eqnarray}
\end{lem}

\begin{lem}\label{lem:slope}\cite{Misiurewicz80}
Given a continuous function $f:[a,b]\to [a,b]$, which is piece-wise monotone, if $f$ is
affine with the slope coefficient of absolute value s on each piece of monotonicity, then 
\begin{eqnarray}
h_{top}(f)=
\max\set{0,\log_2 s}.
\end{eqnarray}
\end{lem}

\begin{lem}\label{lem:sub}\cite{Alseda00}
Given a subadditive sequence $\set{a_k}^{\infty}_{k=1}$(i.e. $a_{n+k}\leq a_n+a_k$), we have 
\begin{eqnarray}
\lim_{k\to \infty}\frac{a_k}{k}
\end{eqnarray}
exists and is equal to $\inf_k\frac{a_k}{k} $.
\end{lem}

\end{appendix}


\begin{thebibliography}{99}


\bibitem[ABMM16]{Arora2016}
Raman Arora, Amitabh Basu, Poorya Mianjy, and Anirbit Mukherjee,
{Understanding deep neural networks with rectified linear units},
{arXiv:1611.01491.}

\bibitem[AKM65]{Adler65}
R. L. Adler, A. G. Konheim and M. H. McAndrew,
{Topological entropy},
{\emph{Transactions of the American Mathematical Society },
\textbf{114}(1965), 309--319.
}

\bibitem[ALM00]{Alseda00}
Llu\'is Alsed\`a, Jaume Llibre, and Micha\l~Misiurewicz,
{Combinatorial dynamics and entropy in dimension one},
(2000).




\bibitem[Bar94]{Barron1994}
Andrew R. Barron,
{Approximation and estimation bounds for artificial neural networks},
{\emph{Machine Learning},
\textbf{14(1)}(1994), 115--133.
}


\bibitem[CNPW19]{Chatziafratis2019}
Vaggos Chatziafratis, Sai Ganesh Nagarajan, Ioannis Panageas, and Xiao Wang,
{Depth-width trade-offs for relu networks via sharkovsky's theorem},
{arXiv:1912.04378}.



\bibitem[CNP20]{Chatziafratis2020}
Vaggos Chatziafratis, Sai Ganesh Nagarajan, Ioannis Panageas,
{Better depth-width trade-offs for neural networks through the lens of dynamical systems},
{arXiv:2003.00777}.







\bibitem[Cyb89]{Cybenko1989}
 George Cybenko, 
 {Approximation by superpositions of a sigmoidal function},
  {\emph {Mathematics of Control, Signals and Systems}, 
  \textbf{2(4)}(1989), 303--314.
  }


\bibitem[DB11]{Shawe11}
Olivier Delalleau and Yoshua Bengio,
{Shallow vs. deep sum-product networks},
{In \emph{ Advances in Neural Information Processing Systems},
(2011), 666-674.
}



\bibitem[ES16]{Eldan16}
Ronen Eldan and Ohad Shamir,
{The power of depth for feedforward neural networks},
{In \emph{Conference of learning theory},
(2016), 907--940.
}

\bibitem[Fun89]{Funahashi1989}
Ken-Ichi Funahashi,
{On the approximate realization of continuous mappings by neural networks},
{\emph{Neural Networks},
\textbf{2(3)}(1989), 183--192.
}

\bibitem[Has86]{Hastad1986}
John Hastad,
{Almost optimal lower bounds for small depth circuits}, 
{In \emph{Proceedings of the
eighteenth annual ACM symposium on Theory of computing},
ACM (1986), 6--20.
}


  
 

\bibitem[H\r{a}s87]{Hastad87}
Johan H\r{a}stad, 
{Computational limitations of small-depth circuits},
{(1987), MIT Press. 
}

\bibitem[HMW89]{Hornik1989}
Kurt Hornik, Maxwell, Stinchcombe and Halbert White,
{Multilayer feedforward networks are universal approximators},
{\emph{Neural Networks},
\textbf{2(5)}(1989), 359--366.
}




\bibitem[KTB19]{Kileel2019}
Joe Kileel, Matthew Trager, and Joan Bruna,
{On the expressive power of deep polynomial neural networks},
{arXiv:1905.12207.}





\bibitem[LS16]{Liang2016}
Shiyu Liang and Rayadurgam Srikant,
{Why deep neural networks for function approximation?},
{ arXiv:1610.04161.}



\bibitem[Mis79]{Misiurewicz79hor}
Micha\l~Misiurewicz, 
{Horseshoes for mappings of an interval},
 {\emph{Bull. Acad.
  Pol. Sci., Ser. Sci. Math.},
  \textbf{27}(1979), 167--169.
 }


\bibitem[Mis80a]{Misiurewicz80hor}
Micha\l~Misiurewicz, 
{Horseshoes for continuous mappings of an interval},
{\emph{Dynamical systems},
(1980), 127--135.
}


\bibitem[Mis80b]{Misiurewicz80}
Micha\l~Misiurewicz and Wies\l aw Szlenk, 
{Entropy of piecewise monotone mappings},
{\emph{ Studia Math},
\textbf{67}(1980), 45--63.
}





\bibitem[MM14]{Martens2014}
James Martens and Venkatesh Medabalimi,
{On the expressive efficiency of sum product networks},
{arXiv:1411.7717.}



\bibitem[MPCB14]{Montufar2014}
Guido F. Montufar, Razvan Pascanu, Kyunghyun Cho, and Yoshua Bengio,
{On the number of linear regions of deep neural networks},
{In \emph{Advances in Neural Information Processing Systems},
(2014), 2294--2932.
}

\bibitem[MSS19]{Malach2019}
 Eran Malach and Shai Shalev-Shwartz, 
 {Is deeper better only when shallow is good?},
 {arXiv:1903.03488.
 }





\bibitem[PGM94]{Parberry94}
Ian Parberry, Michael R. Garey, and Albert Meyer,
{Circuit complexity and neural networks},
{(1984), MIT Press.}



\bibitem[PLRDG16]{Poole2016}
Ben Poole, Subhaneil Lahiri, Maithra Raghu, Jascha Sohl-Dickstein, and Surya Ganguli,
{Exponential expressivity in deep neural networks through transient chaos. In Advances in neural information processing systems},
{In \emph{Advances in Neural Information Processing Systems},
(2016), 3360--3368.
}



\bibitem[RPJKGSD17]{Raghu17}
Maithra Raghu, Ben Poole, Jon Kleinberg, Surya Ganguli, and Jascha Sohl Dickstein,
{On the expressive power of deep neural networks},
{In \emph{Proceedings of the 34th International Conference on Machine Learning},
\textbf{70}(2017), 2847--2854.
}



\bibitem[RST15]{RossmanFOCS15}
Benjamin Rossman, Rocco A. Servedio, and Li-Yang Tan,
{An average-case depth hierarchy theorem for boolean circuits},
{In \emph{Proceedings of the 2015 IEEE 56th  Annual Symposium on Foundations of Computer Science (FOCS)},
IEEE(2015),1030--1048.
}



\bibitem[Sch00]{Schmitt1999}
Michael Schmitt,
{Lower bounds on the complexity of approximating continuous functions by sigmoidal neural networks},
{In \emph{Advances in Neural Information Processing Systems},
(2000), 328--334.
}

\bibitem[Sha64]{Sharkovsky64}
OM Sharkovsky,
{Coexistence of the cycles of a continuous mapping of the line into itself},
{\emph{Ukrainskij matematicheskij zhurnal},
\textbf{16(01)} (1964), 61--71.
}


\bibitem[Sha65]{Sharkovsky65}
OM Sharkovsky,
{On cycles and structure of continuous mapping},
{\emph{Ukrainskij matematicheskij zhurnal},
\textbf{17(03)} (1965), 104--111.
}




\bibitem[Tel15]{Telgarsky2015}
Matus Telgarsky,
{Representation benefits of deep feedforward networks},
{arXiv:1509.08101.}

\bibitem[Tel16]{Telgarsky2016}
Matus Telgarsky, 
{benefits of depth in neural networks},
{In \emph{Conference on Learning Theory},
(2016), 1517--1539.
}






\bibitem[You81]{Young81}
Lai-Sang Young,
{On the prevalence of horseshoes},
{\emph{Transactions of the American Mathematical Society},
\textbf{263}(1981), 75--88.
}


\end{thebibliography}
\end{document}